\documentclass{article}

\usepackage{times}
\usepackage{graphicx} 
\usepackage{mathtools}
\usepackage[dvipsnames]{xcolor}
\usepackage{subfigure}
\usepackage{macros}
\usepackage{thmtools}
\usepackage{thm-restate}
\usepackage{amsopn,amssymb,amsthm,amsmath,bbm,bm}
\usepackage{multirow}
\usepackage{booktabs}

\usepackage{natbib}

\usepackage{algorithm}
\usepackage[noend]{algorithmic}

\newcommand{\Cov}{\operatorname{Cov}}

\usepackage{hyperref}

\usepackage[accepted]{icml2017}

\icmltitlerunning{Learning the Structure of Generative Models without Labeled Data}

\begin{document} 

\twocolumn[
\icmltitle{Learning the Structure of Generative Models without Labeled Data}



\icmlsetsymbol{equal}{*}

\begin{icmlauthorlist}
\icmlauthor{Stephen H. Bach}{stanford}
\icmlauthor{Bryan He}{stanford}
\icmlauthor{Alexander Ratner}{stanford}
\icmlauthor{Christopher R\'e}{stanford}
\end{icmlauthorlist}

\icmlaffiliation{stanford}{Stanford University, Stanford, California}

\icmlcorrespondingauthor{Stephen Bach}{bach@cs.stanford.edu}

\icmlkeywords{structure learning, data programming, weak supervision}

\vskip 0.3in
]



\printAffiliationsAndNotice{}  


\begin{abstract} 

Curating labeled training data has become the primary bottleneck in machine learning.
Recent frameworks address this bottleneck with generative models to synthesize labels at scale from weak supervision sources.
The generative model's dependency structure directly affects the quality of the estimated labels, but selecting a structure automatically without any labeled data is a distinct challenge.
We propose a structure estimation method that maximizes the $\ell_1$-regularized marginal pseudolikelihood of the observed data.
Our analysis shows that the amount of unlabeled data required to identify the true structure scales sublinearly in the number of possible dependencies for a broad class of models.
Simulations show that our method is 100$\times$ faster than a maximum likelihood approach and selects $1/4$ as many extraneous dependencies.
We also show that our method provides an average of 1.5 F1 points of improvement over existing, user-developed information extraction applications on real-world data such as PubMed journal abstracts.
\end{abstract}


\section{Introduction}

Supervised machine learning traditionally depends on access to labeled training data, a major bottleneck in developing new methods and applications.
In particular, deep learning methods require tens of thousands or more labeled data points for each specific task.
Collecting these labels is often prohibitively expensive, especially when specialized domain expertise is required, and major technology companies are investing heavily in hand-curating labeled training data \citep{wired16, time17}.
Aiming to overcome this bottleneck, there is growing interest in using generative models to synthesize training data from weak supervision sources such as heuristics, knowledge bases, and weak classifiers trained directly on noisy sources.
Rather than treating training labels as gold-standard inputs, such methods model training set creation as a process in order to generate training labels at scale.
The true class label for a data point is modeled as a latent variable that generates the observed, noisy labels.
After fitting the parameters of this generative model on unlabeled data, a distribution over the latent, true labels can be inferred.

The structure of such generative models directly affects the inferred labels, and prior work assumes that the structure is user-specified \citep{alfonseca:acl12, takamatsu:acl12, roth:emnlp13,ratner:nips16}.
One option is to assume that the supervision sources are conditionally independent given the latent class label.
However, statistical dependencies are common in practice, and not taking them into account leads to misjudging the accuracy of the supervision.
We cannot rely in general on users to specify the structure of the generative model, because supervising heuristics and classifiers might be independent for some data sets but not others.
We therefore seek an efficient method for automatically learning the structure of the generative model from weak supervision sources alone.

While structure learning in the supervised setting is well-studied \citep[e.g.,][see also Section~\ref{sec:related}]{meinshausen:annalsofstats06, zhao:jmlr06, ravikumar:annalsofstats10}, learning the structure of generative models for weak supervision is challenging because the true class labels are latent.
Although we can learn the parameters of generative models for a given structure using stochastic gradient descent and Gibbs sampling, modeling all possible dependencies does not scale as an alternative to model selection.
For example, estimating all possible correlations for a modestly sized problem of 100 weak supervision sources takes over 40 minutes.
(For comparison, our proposed approach solves the same problem in 15 seconds.)
As users develop their supervision heuristics, rerunning parameter learning to identify dependencies becomes a prohibitive bottleneck.

We propose an estimator to learn the dependency structure of a generative model without using any labeled training data.
Our method maximizes the $\ell_1$-regularized marginal pseudolikelihood of each supervision source's output independently, selecting those dependencies that have nonzero weights.
This estimator is analogous to maximum likelihood for logistic regression, except that we marginalize out our uncertainty about the latent class label.
Since the pseudolikelihood is a function of one free variable and marginalizes over one other variable, we compute the gradient of the marginal pseudolikelihood exactly, avoiding the need for approximating the gradient with Gibbs sampling, as is done for maximum likelihood estimation.

Our analysis shows that the amount of data required to identify the true structure scales sublinearly in the number of possible dependencies for a broad class of models.
Intuitively, this follows from the fact that learning the generative model's parameters is possible when there are a sufficient number of better-than-random supervision sources available.
With enough signal to estimate the latent class labels better than random guessing, those estimates can be refined until the model is identified.

We run experiments to confirm these predictions.
We also compare against the alternative approach of considering all possible dependencies during parameter learning.
We find that our method is 100$\times$ faster.
In addition, our method returns $1/4$ as many extraneous correlations on synthetic data when tuned for comparable recall.
Finally, we demonstrate that on real-world applications of weak supervision, using generative models with automatically learned dependencies improves performance.
We find that our method provides on average 1.5 F1 points of improvement over existing, user-developed information extraction applications on PubMed abstracts and hardware specification sheets.


\section{Background}

When developing machine learning systems, the primary bottleneck is often curating a sufficient amount of labeled training data.
Hand labeling training data is expensive, time consuming, and often requires specialized knowledge.
Recently researchers have proposed methods for synthesizing labels from noisy label sources using generative models.
(See Section~\ref{sec:related} for a summary.)
We ground our work in one framework, data programming \cite{ratner:nips16}, that generalizes many approaches in the literature.

In data programming, weak supervision sources are encoded as \emph{labeling functions}, heuristics that label data points (or abstain).
A generative probabilistic model is fit to estimate the accuracy of the labeling functions and the strength of any user-specified statistical dependencies among their outputs.
In this model, the true class label for a data point is a latent variable that generates the labeling function outputs.
After fitting the parameters of the generative model, a distribution over the latent, true labels can be estimated and be used to train a discriminative model by minimizing the expected loss with respect to that distribution.

We formally describe this setup by first specifying for each data point $\x_\iy$ a latent random variable $\y_\iy \in \{-1, 1\}$ that is its true label.
For example, in an information extraction task, $\x_\iy$ might be a span of text.
Then, $\y_\iy$ can represent whether it is a mention of a company or not (entity tagging).
Alternatively, $\x_\iy$ might be a more complex structure, such as a tuple of canonical identifiers along with associated mentions in a document, and then $\y_\iy$ can represent whether a relation of interest over that tuple is expressed in the document (relation extraction).

We do not have access to $\y_\iy$ (even at training time), but we do have $\nlf$ user-provided labeling functions $\lf_1, \dots, \lf_\nlf$ that can be applied to $\x_\iy$ to produce outputs $\lfout_{\iy 1}, \dots, \lfout_{\iy \nlf}$.
For example, for the company-tagging task mentioned above, a labeling function might apply the regular expression {\tt .+\textbackslash sInc\textbackslash.} to a span of text and return whether it matched.
The domain of each $\lfout_{\iy \ilf}$ is $\{-1, 0, 1 \}$, corresponding to \emph{false}, \emph{abstaining}, and \emph{true}.
Generalizing to the multiclass case is straightforward.

Our goal is to estimate a probabilistic model that generates the labeling-function outputs $\lfout \in \{-1, 0, 1\}^{\ny \times \nlf}$.
A common assumption is that the outputs are conditionally independent given the true label, and that the relationship between $\lfout$ and $\y$ is governed by $\nlf$ \emph{accuracy} dependencies
\[
\dep^\acc_\ilf (\lfout_\iy, \y_\iy) \defeq ~ y_\iy \lfout_{\iy \ilf}
 \]
with a parameter $\param^\acc_\ilf$ modeling how accurate each labeling function $\lf_\ilf$ is.
We refer to this structure as the \emph{conditionally independent model}, and specify it as
\begin{equation}
\label{eq:ci}
\p_\param(\lfout, \Y) \propto \exp \left( \sum_{\iy = 1}^\ny \sum_{\ilf = 1}^\nlf \param^\acc_\ilf \dep^\acc_\ilf (\lfout_\iy, \y_\iy) \right)~,
\end{equation}
where $\Y \defeq \y_1, \dots, \y_\ny$.

We estimate the parameters $\param$ by minimizing the negative log marginal likelihood $\p_\param(\lfoutobs)$ for an observed matrix of labeling function outputs $\lfoutobs$:
\begin{equation}
\label{eq:mle}
\argmin_\param ~ - \log \sum_\Y \p_\param(\lfoutobs, \Y)~.
\end{equation}

Optimizing the likelihood is straightforward using stochastic gradient descent.
The gradient of objective~(\ref{eq:mle}) with respect to parameter $\param^\acc_\ilf$ is 
\[
\sum_{\iy = 1}^\ny \left(
\E_{\lfout, \Y \sim \param} \left[ \dep^\acc_\ilf (\lfout_\iy, \y_\iy) \right]
-
\E_{\Y \sim \param | \lfoutobs} \left[ \dep^\acc_\ilf (\lfoutobs_\iy, \y_\iy) \right]
\right)~,
\]
the difference between the corresponding sufficient statistic of the joint distribution $p_\param$ and the same distribution conditioned on $\lfoutobs$.
In practice, we can interleave samples to estimate the gradient and gradient steps very tightly, taking a small step after each sample of each variable $\lfout_{\iy \ilf}$ or $\y_\iy$, similarly to contrastive divergence \cite{hinton:neuralcomp02}.

The conditionally independent model is a common assumption, and using a more sophisticated generative model currently requires users to specify its structure.
In the rest of the paper, we address the question of automatically identifying the dependency structure from the observations $\lfoutobs$ without observing $\Y$.


\section{Structure Learning without Labels}
\label{sec:method}

Statistical dependencies arise naturally among weak supervision sources.
In data programming, users often write labeling functions with directly correlated outputs or even labeling functions deliberately designed to reinforce others with narrow, more precise heuristics.
To address this issue, we generalize the conditionally independent model as a factor graph with additional dependencies, including higher-order factors that connect multiple labeling function outputs for each data point $\x_\iy$ and label $\y_\iy$.
We specify the general model as
\begin{equation}
\label{eq:general}
\p_\param(\lfout, \Y) \propto \exp \left( \sum_{\iy = 1}^\ny \sum_{\deptype \in \alldeptypes} \sum_{\lfsubset \in \allsubsets_\deptype} \param^\deptype_\lfsubset \dep^\deptype_\lfsubset(\lfout_\iy, \y_\iy) \right)~.
\end{equation}
Here $\alldeptypes$ is the set of dependency types of interest, and $\allsubsets_\deptype$ is a set of index tuples, indicating the labeling functions that participate in each dependency of type $\deptype \in \alldeptypes$.
We start by defining standard \emph{correlation} dependencies of the form
\[
\dep^\cor_{\ilf \ilfalt} (\lfout_\iy, \y_\iy) \defeq ~ \mathbbm{1}\{ \lfout_{\iy \ilf} = \lfout_{\iy \ilfalt} \}~.
\]
We refer to such dependencies as pairwise among labeling functions because they depend only on two labeling function outputs.
We can also consider higher-order dependencies that involve more variables, such as \emph{conjunction} dependencies of the form
\[
\dep^\text{And}_{\ilf \ilfalt} (\lfout_\iy, \y_\iy) \defeq ~ \mathbbm{1}\{ \lfout_{\iy \ilf} = \y_\iy \wedge \lfout_{\iy \ilfalt} = \y_\iy\}~.
\]

Estimating the structure of the distribution $\p_\param(\lfout, \Y)$ is challenging because $\Y$ is latent; we never observe its value, even during training.
We must therefore work with the marginal likelihood $\p_\param(\lfout)$.
Learning the parameters of the generative model jointly requires Gibbs sampling to estimate gradients.
As the number of possible dependencies increases at least quadratically in the number of labeling functions, this heavyweight approach to learning does not scale (see Section~\ref{sec:comparison}).

\subsection{Learning Objective}

We can scale up learning over many potentially irrelevant dependencies by optimizing a different objective: the log marginal pseudolikelihood of the outputs of a single labeling function $\lf_\ilf$, i.e., conditioned on the outputs of the others $\lf_{\setminus \ilf}$, using $\ell_1$ regularization to induce sparsity.
The objective is
\begin{align}
& \argmin_{\param} ~ -\log \p_\param(\lfoutobs_\ilf \mid \lfoutobs_{\setminus \ilf}) + \threshold \|\param\|_1 \label{eq:pseudo} \\
= ~ &\argmin_{\param} ~ -\sum_{\iy=1}^\ny  \log \sum_{\y_\iy} \p_\param(\lfoutobs_{\iy \ilf}, \y_\iy \mid \lfoutobs_{\iy \setminus \ilf}) + \threshold \|\param\|_1, \notag
\end{align}
where $\threshold > 0$ is a hyperparameter.

By conditioning on all other labeling functions in each term $\log \sum_{\y_\iy} \p_\param(\lfoutobs_{\iy \ilf}, \y_\iy \mid \lfoutobs_{\iy \setminus \ilf})$, we ensure that the gradient can be computed in polynomial time with respect to the number of labeling functions, data points, and possible dependencies; without requiring any sampling or variational approximations.
The gradient of the log marginal pseudolikelihood is the difference between two expectations: the sufficient statistics conditioned on all labeling functions but $\lf_\ilf$, and conditioned on all labeling functions:
\begin{align}
& - \frac{\partial \log \p(\lfoutobs_\ilf \mid \lfoutobs_{\setminus \ilf})}{\partial \param^\deptype_\lfsubset} = \alpha - \beta, \label{eq:gradient}
\end{align}
where
\begin{align*}
\alpha &\defeq \sum_{\iy=1}^\ny \sum_{\lfout_{\iy \ilf}, \y_\iy} \p_\param(\lfout_{\iy \ilf}, \y_\iy \mid \lfoutobs_{\iy \setminus \ilf}) \dep^\deptype_\lfsubset((\lfout_{\iy \ilf}, \lfoutobs_{\iy \setminus \ilf}), \y_\iy) \notag \\
\beta &\defeq \sum_{\iy=1}^\ny \sum_{\y_\iy} \p(\y_\iy \mid \lfoutobs_\iy) \dep^\deptype_\lfsubset(\lfoutobs_\iy, \y_\iy). \notag
\end{align*}
Note that in the definition of $\alpha$, $\dep^\deptype_\lfsubset$ operates on the value of $\lfout_{\iy \ilf}$ set in the summation and the observed values of $\lfoutobs_{\iy \setminus \ilf}$.

We optimize for each labeling function $\lf_\ilf$ in turn, selecting those dependencies with parameters that have a sufficiently large magnitude and adding them to the estimated structure.

\subsection{Implementation}

\begin{algorithm}[tb]
   \caption{Structure Learning for Data Programming}
   \label{alg:method}
\begin{algorithmic}
   \vskip 0.2em
   \STATE {\bfseries Input:} Observations $\lfoutobs \in \{-1, 0, 1\}^{\ny \times \nlf}$, threshold $\threshold$,
   distribution $p$ with parameters $\param$, initial parameters $\param^0$, \\ step size $\stepsize$, epoch count $\epochs$, truncation frequency $\truncation$
   \vskip 0.4em
   \STATE $\depset \leftarrow \emptyset$
   \FOR{$\ilf = 1$ {\bfseries to} $\nlf$}
   \STATE $\param \leftarrow \param^0$
   \FOR{$\epoch = 1$ {\bfseries to} $\epochs$}
   \FOR{$\iy = 1$ {\bfseries to} $\ny$}
   \vskip 0.4em
   \FOR{$\param^\deptype_\lfsubset$ {\bfseries in} $\param$}
   \vskip 0.1em
   \STATE $\alpha \leftarrow \sum_{\lfout_{\iy \ilf}, \y_\iy} \p(\lfout_{\iy \ilf}, \y_\iy | \lfoutobs_{\iy \setminus \ilf}) \dep^\deptype_\lfsubset((\lfout_{\iy \ilf}, \lfoutobs_{\iy \setminus \ilf}), \y_\iy)$
   \vskip 0.1em
   \STATE $\beta \leftarrow \sum_{\y_\iy} \p(\y_\iy \mid \lfoutobs_\iy) \dep^\deptype_\lfsubset(\lfoutobs_\iy, \y_\iy)$
   \vskip 0.1em
   \STATE $\param^\deptype_\lfsubset \leftarrow \param^\deptype_\lfsubset - \stepsize (\alpha  - \beta)$
   \vskip 0.1em
   \ENDFOR
   \vskip 0.4em
   \IF {$\epoch \ny + \iy$ {\bfseries mod} $\truncation$ {\bfseries is} $0$}
   \FOR{$\param^\deptype_\lfsubset$ {\bfseries in} $\param$ {\bfseries where} $\param^\deptype_\lfsubset > 0$}
   \STATE $\param^\deptype_\lfsubset \leftarrow \max \{0, \param^\deptype_\lfsubset - \truncation \stepsize \threshold \}$
   \ENDFOR
   \FOR{$\param^\deptype_\lfsubset$ {\bfseries in} $\param$ {\bfseries where} $\param^\deptype_\lfsubset < 0$}
   \STATE $\param^\deptype_\lfsubset \leftarrow \min \{0, \param^\deptype_\lfsubset + \truncation \stepsize \threshold \}$
   \ENDFOR
   \ENDIF
   \vskip 0.4em
   \ENDFOR
   \ENDFOR
   \vskip 0.4em
   \FOR{$\param^\deptype_\lfsubset$ {\bfseries in} $\param$ {\bfseries where} $\ilf \in \lfsubset$}
   \vskip 0.1em
   \IF{$| \param^\deptype_{\lfsubset} | > \threshold$} 
   \vskip 0.1em
   \STATE $\depset \leftarrow \depset \cup \{(\lfsubset, \deptype) \}$
   \ENDIF
   \ENDFOR
   \ENDFOR
   \RETURN $\depset$
\end{algorithmic}
\end{algorithm}

We implement our method as Algorithm~\ref{alg:method}, a stochastic gradient descent (SGD) routine.
At each step of the descent, the gradient~(\ref{eq:gradient}) is estimated for a single data point, which can be computed in closed form.
Using SGD has two advantages.
First, it requires only first-order gradient information.
Other methods for $\ell_1$-regularized regression like interior-point methods \citep{koh:jmlr07} usually require computing second-order information.
Second, the observations $\lfoutobs$ can be processed incrementally.
Since data programming operates on unlabeled data, which is often abundant, scalability is crucial.
To implement $\ell_1$ regularization as part of SGD, we use an online truncated gradient method \citep{langford:jmlr09}.

In practice, we find that the only parameter that requires tuning is $\threshold$, which controls the threshold and regularization strength.
Higher values induce more sparsity in the selected structure.
For the other parameters, we use the same values in all of our experiments: step size $\stepsize = {\ny}^{-1}$, epoch count $\epochs = 10$, and truncation frequency $\truncation = 10$.


\section{Analysis}
\label{sec:analysis}

We provide guarantees on the probability that Algorithm \ref{alg:method} successfully recovers the exact dependency structure.
We first provide a general recovery guarantee for all types of possible dependencies, including both pairwise and higher-order dependencies.
However, in many cases, higher-order dependencies are not necessary to model the behavior of the labeling functions.
In fact, as we demonstrate in Section~\ref{sec:realworld}, in many useful models there are only accuracy dependencies and pairwise correlations.
In this case, we show as a corollary to our general result that the number of samples required is sublinear in the number of possible dependencies, specifically $O(\nlf \log \nlf)$.

Previous analyses for the supervised case do not carry over to the unsupervised setting because the problem is no longer convex.
For example, analysis of an analogous method for supervised Ising models \citep{ravikumar:annalsofstats10} relies on Lagrangian duality and a tight duality gap, which does not hold for our estimation problem.
Instead, we reason about a region of the parameter space in which we can estimate $\Y$ well enough that we can eventually approach the true model.

We now state the conditions necessary for our guarantees.
First are two standard conditions that are needed to guarantee that the dependency structure can be recovered with any number of samples.
One, we must have some set $\Theta\subset\mathbf{R}^M$ of feasible parameters.
Two, the true model is in $\Theta$, i.e., there exists some choice of $\param^*\in\Theta$ such that
\begin{equation} \label{eq: well-specified}
\begin{split}
&\pi^*(\lfout, \Y) = \p_{\param^*}(\lfout, \Y), \\
  &\forall \lfout\in\{-1,0,1\}^{\ny \times \nlf},\Y\in\{-1,1\}^\ny
\end{split}
\end{equation}
where $\pi^*$ is the true distribution.

Next, let $\Dep_\ilf$ denote the set of dependencies that involve either labeling function $\lf_\ilf$ or the true label $\y$.
For any feasible parameter $\theta\in\Theta$ and $\ilf\in \{1,\ldots, \nlf\}$, there must exist $c>0$ such that
\begin{equation} \label{eq: cov}
  \begin{split}
    cI &{}+ \sum_{\iy=1}^{\ny}\Cov_{(\lfout, Y)\sim p_{\theta}}(\Dep_\ilf(\lfout, Y)\mid \lfout_\iy = \lfoutobs_\iy) \\
  &\preceq
    \sum_{\iy=1}^{\ny}\Cov_{(\lfout, Y)\sim p_{\theta}}(\Dep_\ilf(\lfout, Y)\mid \lfout_{\iy\setminus \ilf} = \lfoutobs_{\iy\setminus \ilf}). \\
  \end{split}
\end{equation}
This means that for each labeling function, we have a better estimate of the dependencies with the labeling function than without.
It is analogous to assumptions made to analyze parameter learning in data programming.

Finally, we require that all non-zero parameters be bounded away from zero.
That is, for all $\param_i\neq 0$, and some $\kappa > 0$, we have that
\begin{equation}
  \label{eq: notzero}
  |\param_i| \geq \kappa.
\end{equation}

Under these conditions, we are able to provide guarantees on the probability of finding the correct dependency structure.
First, we present guarantees for all types of possible dependencies in Theorem \ref{thmsparsistency}, proved in Appendix~\ref{app: thm}.
For this theorem, we define $\ndep_\ilf$ to be the number of possible dependencies involving either $\lfout_\ilf$ or $\y$, and we define $\ndep$ as the largest of $\ndep_1,\ldots,\ndep_\nlf$.
\begin{restatable}{theorem}{thmsparsistency}
  \label{thmsparsistency}
  Suppose we run Algorithm \ref{alg:method} on a problem where conditions 
  \eqref{eq: well-specified}, \eqref{eq: cov}, and \eqref{eq: notzero} are satisfied.
  Then, for any $\delta > 0$, an unlabeled input dataset of size
  \begin{align*}
  m \geq \frac{32\ndep}{c^2\kappa^2}\log\left(\frac{2n\ndep}{\delta}\right)
  \end{align*}
  is sufficient to recover the exact dependency structure with a probability of at least $1 - \delta$.
\end{restatable}

For general dependencies, $d$ can be as large as the number of possible dependencies due to the fact that higher-order dependencies can connect the true label and many labeling functions.
The rate of Theorem~\ref{thmsparsistency} rate is therefore not directly comparable to that of \citet{ravikumar:annalsofstats10}, which applies to Ising models with pairwise dependencies.

As we demonstrate in Section~\ref{sec:realworld}, however, real-world applications can be improved by modeling just pairwise correlations among labeling functions.
If only considering these dependencies, then $\ndep$ will only be $2\nlf - 1$, rather than the number of potential dependencies.
In Corollary \ref{corsublinear}, we show that a number of samples needed in this case is $O(\nlf\log\nlf)$.
Notice that this is sublinear in the number of possible dependencies, which is $O(n^2)$.
\begin{restatable}{corollary}{corsublinear}
  \label{corsublinear}
  Suppose we run Algorithm \ref{alg:method} on a problem where conditions 
  \eqref{eq: well-specified}, \eqref{eq: cov}, and \eqref{eq: notzero} are satisfied.
  Additionally, assume that the only potential dependencies are accuracy and correlation dependencies.
  Then, for any $\delta > 0$, an unlabeled input dataset of size
  \begin{align*}
    m \geq \frac{64\nlf}{c^2\kappa^2}\log\left(\frac{4\nlf}{\delta}\right)
  \end{align*}
  is sufficient to recover the exact dependency structure with a probability of at least $1 - \delta$.
\end{restatable}
In this case, we see the difference in analyses between the unsupervised and supervised settings.
Whereas the rate of Corollary~\ref{corsublinear} depends on the maximum number of dependencies that could affect a variable in the model class, the rate of \citet{ravikumar:annalsofstats10} depends cubically on the maximum number of dependencies that actually affect any variable in the true model and only logarithmically in the maximum possible degree.
In the supervised setting, the guaranteed rate is therefore tighter for very sparse models.
However, as we show in Section~\ref{sec:samplecomplexity}, the guaranteed rates in both settings are pessimistic, and in practice they appear to scale at the same rate.


\section{Experiments}
\label{sec:experiments}

We implement our method as part of the open source framework Snorkel\footnote{\url{snorkel.stanford.edu}} and evaluate it in three ways.
First, we measure how the probability of returning the exact correlation structure is affected by the problem parameters using synthetic data, confirming our analysis that its sample complexity is sublinear in the number of possible dependencies.
In fact, we find that in practice the sample complexity is lower than the theoretically guaranteed rate, matching the rate seen in practice for fully supervised structure learning.
Second, we compare our method to estimating structures via parameter learning over all possible dependencies.
We demonstrate using synthetic data that our method is 100$\times$ faster and more accurate, selecting $1/4$ as many extraneous correlations on average.
Third, we apply our method to real-world applications built using data programming, such as information extraction from PubMed journal abstracts and hardware specification sheets.
In these applications, users did not specify any dependencies between the labeling functions they authored;
however, as we detail in Section~\ref{sec:realworld}, these dependencies naturally arise, for example due to
explicit composing, relaxing, or tightening of labeling function heuristics; related distant supervision sources;
or multiple concurrent developers writing labeling functions.
We show that learning this structure improves performance over the conditionally independent model, giving an average 1.5 F1 point boost.

\subsection{Sample Complexity}
\label{sec:samplecomplexity}

\begin{figure}[t]
\begin{center}
\centerline{\includegraphics[width=.75\columnwidth]{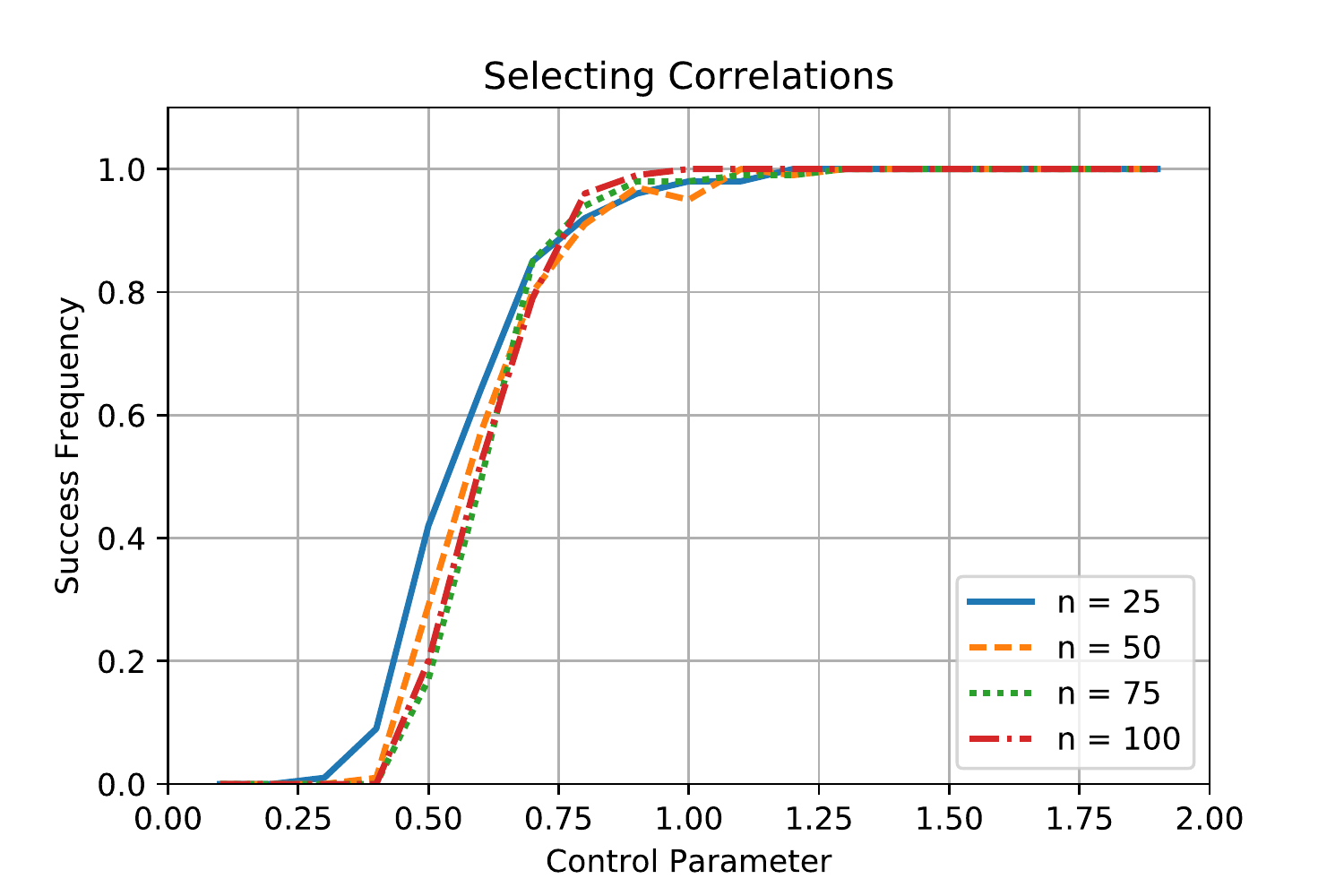}}
\caption{Algorithm~\ref{alg:method} returns the true structure consistently when the control parameter $\control$ reaches 1.0 for the number of samples defined by~(\ref{eq:samples}).
The number of samples required to identify a model in practice scales logarithmically in $\nlf$, the number of labeling functions.}
\label{fig:scaling_lfs}
\end{center}
\vskip -0.2in
\end{figure}

We test how the probability that Algorithm~\ref{alg:method} returns the correct correlation structure depends on the true distribution.
Our analysis in Section~\ref{sec:analysis} guarantees that the sample complexity grows at worst on the order $O(\nlf \log \nlf)$ for $\nlf$ labeling functions.
In practice, we find that structure learning performs better than this guaranteed rate, depending linearly on the number of true correlations and logarithmically on the number of possible correlations.
This matches the observed behavior for fully supervised structure learning for Ising models \citep{ravikumar:annalsofstats10}, which is also tighter than the best known theoretical guarantees.

To demonstrate this behavior, we attempt to recover the true dependency structure using a number of samples defined as
\begin{equation}
\label{eq:samples}
\ny \defeq 750 \cdot \control \cdot \truemaxdeg \cdot \log \nlf
 \end{equation}
where $\truemaxdeg$ is the maximum number of dependencies that affect any one labeling function.
For example, in the conditionally independent model $\truemaxdeg = 1$ and in a model with one correlation $\truemaxdeg = 2$.
We vary the control parameter $\control$ from $0.1$ to $2.0$ to determine the point at which $\ny$ is sufficiently large for Algorithm~\ref{alg:method} to recover the true dependency structure.
(The constant $750$ was selected so that it succeeds with high probability around $\control=1.0$.)

We first test the effect of varying $\nlf$, the number of labeling functions.
For $\nlf \in \{25, 50, 75, 100\}$, we set two pairs of labeling functions to be correlated with $\param^\text{Cor}_{\ilf \ilfalt} = 0.25.$
We set $\param^\text{\acc}_{\ilf} = 1.0$ for all $\ilf$.
We then generate $\ny$ samples for each setting of $\control$ over 100 trials.
Figure~\ref{fig:scaling_lfs} shows the fraction of times Algorithm~\ref{alg:method} returns the correct correlation structure as a function of the control parameter $\control$.
That the curves are aligned for different values of $\nlf$ shows that the sample complexity in practice scales logarithmically in $\nlf$.

\begin{figure}[t]
\begin{center}
\centerline{\includegraphics[width=.75\columnwidth]{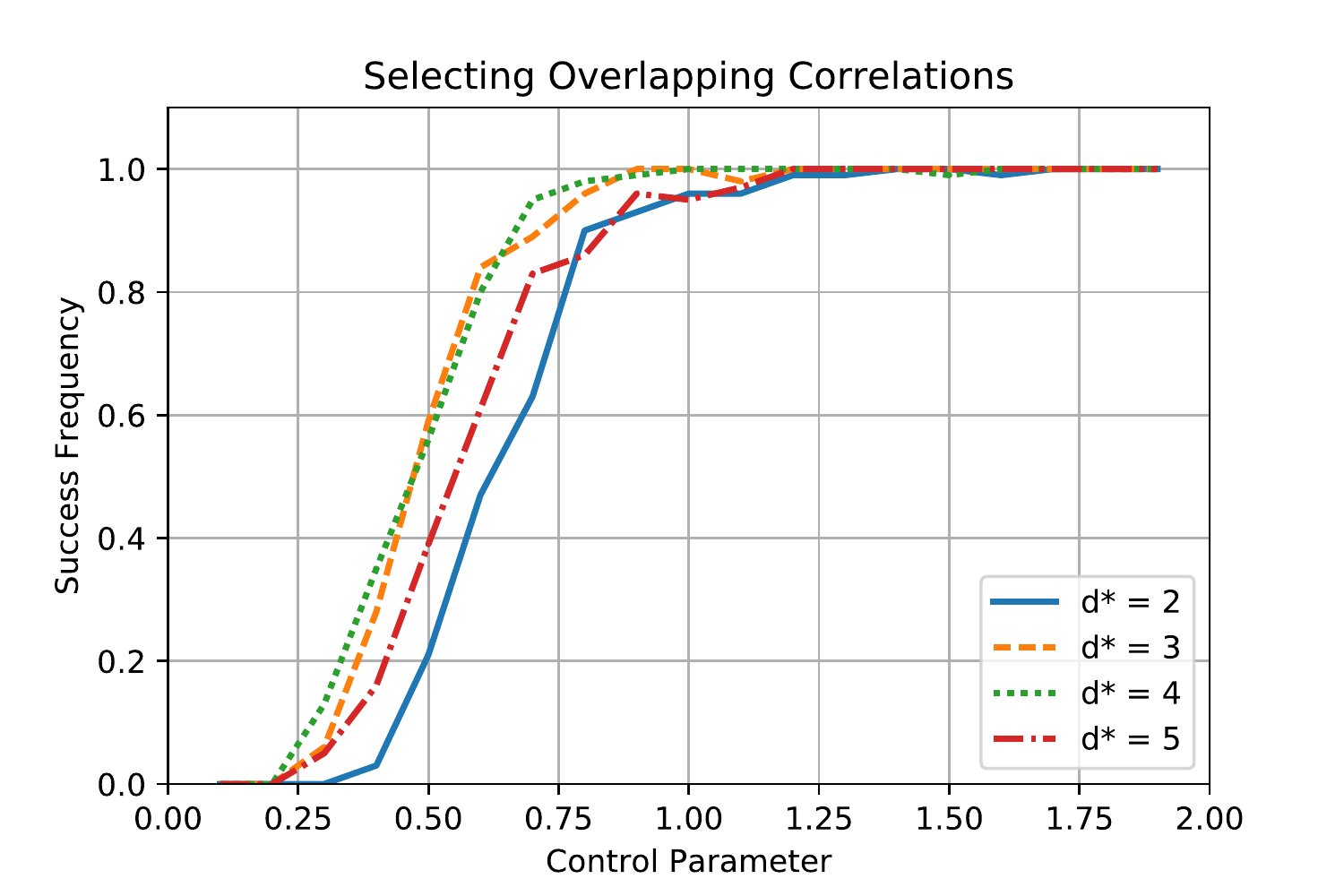}}
\caption{Algorithm~\ref{alg:method} returns the true structure consistently when the control parameter $\control$ reaches 1.0 for the number of samples defined by~(\ref{eq:samples}).
The number of samples required to identify a model in practice scales linearly in $\truemaxdeg$, the maximum number of dependencies affecting any labeling function.}
\label{fig:scaling_overlapping}
\end{center}
\vskip -0.2in
\end{figure}

We next test the effect of varying $\truemaxdeg$, the maximum number of dependencies that affect a labeling function in the true distribution.
For $25$ labeling functions, we add correlations to the true model to form cliques of increasing degree.
All parameters are the same as in the previous experiment.
Figure~\ref{fig:scaling_overlapping} shows that for increasing values of $\truemaxdeg$, (\ref{eq:samples}) again predicts the number of samples for Algorithm~\ref{alg:method} to succeed.
That the curves are aligned for different values of $\truemaxdeg$ shows that the sample complexity in practice scales linearly in $\truemaxdeg$.

\subsection{Comparison with Maximum Likelihood}
\label{sec:comparison}

We next compare Algorithm~\ref{alg:method} with an alternative approach.
Without an efficient structure learning method, one could maximize the marginal likelihood of the observations $\lfoutobs$ while considering all possible dependencies.
To measure the benefits of maximizing the marginal pseudolikelihood, we compare its performance against an analogous maximum likelihood estimation routine that also uses stochastic gradient descent, but instead uses Gibbs sampling to estimate the intractable gradient of the objective.

We create different distributions over $n$ labeling functions by selecting with probability $0.05$ pairs of labeling functions to make correlated.
Again, the strength of correlation is set at $\theta^\text{Cor}_{\ilf \ilfalt} = 0.25$ and accuracy is set at $\theta^\acc_\ilf =1.0$.
We generate 100 distributions for $\nlf \in \{25, 30, 35, \dots, 100\}$.
For each distribution we generate 10,000 samples and attempt to recover the true correlation structure.

\begin{figure}[t]
\begin{center}
\centerline{\includegraphics[width=.75\columnwidth]{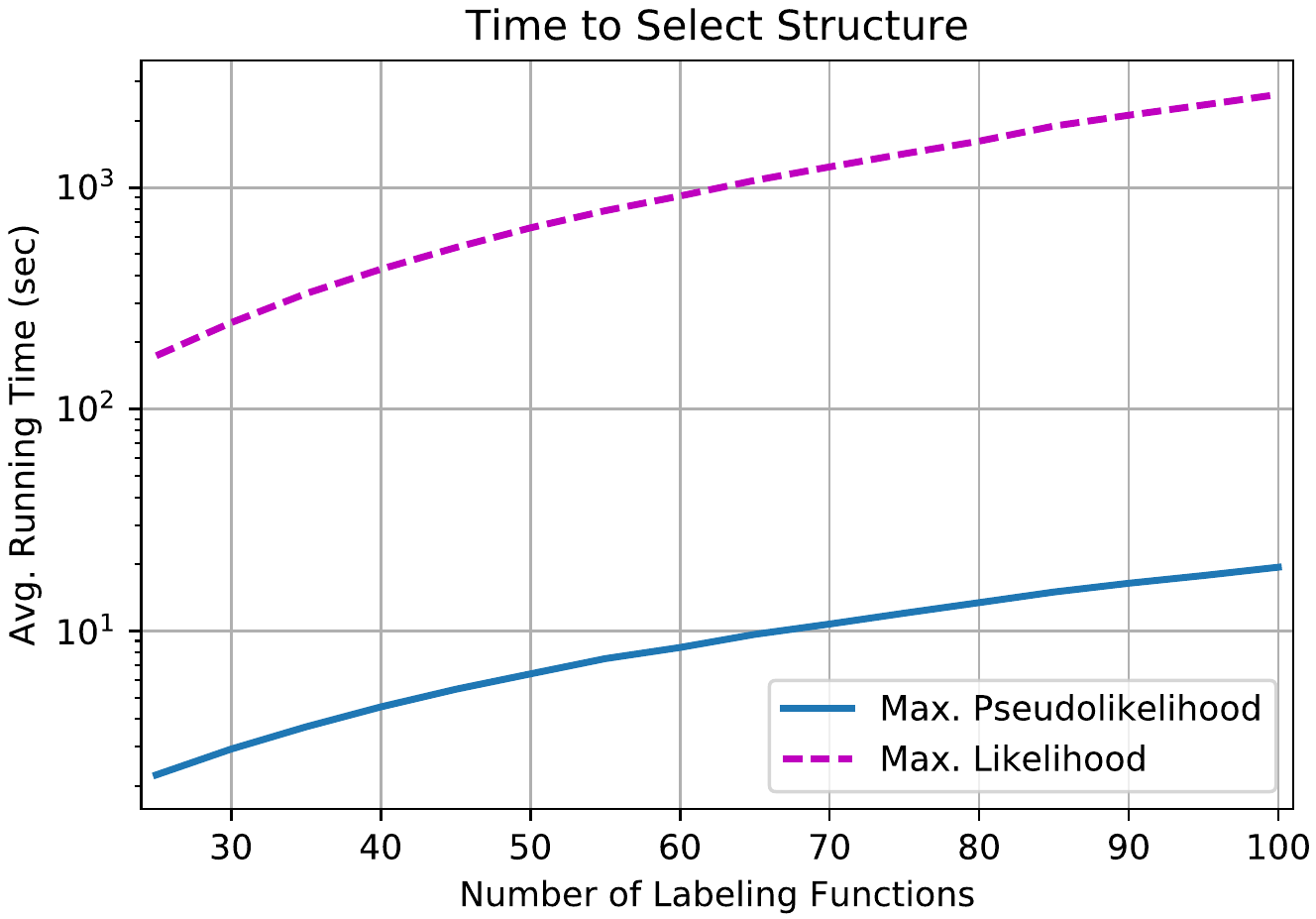}}
\caption{Comparison of structure learning with using maximum likelihood parameter estimation to select a  model structure.
Structure learning is two orders of magnitude faster.}
\label{fig:comparison_speed}
\end{center}
\vskip -0.2in
\end{figure}

We first compare running time between the two methods.
Our implementation of maximum likelihood estimation is designed for speed: for every sample taken to estimate the gradient, a small update to the parameters is performed.
This approach is state-of-the-art for high-speed learning for factor graphs \citep{zhang:vldb14}.
However, the need for sampling the variables $\lfout$ and $\Y$ is still computationally expensive.
Figure~\ref{fig:comparison_speed} shows that by avoiding variable sampling, using pseudolikelihood is 100$\times$ faster.

We next compare the accuracy of the two methods, which depends on the regularization $\threshold$.
The ideal is to maximize the probability of perfect recall with few extraneous correlations, because subsequent parameter estimation can reduce the influence of an extraneous correlation but cannot discover a  missing correlation.
We tune $\threshold$ independently for each method.
Figure~\ref{fig:comparison_accuracy} (top) shows that maximum pseudolikelihood is able to maintain higher levels of recall than maximum likelihood as the problem size increases.
Figure~\ref{fig:comparison_accuracy} (bottom) shows that even tuned for better recall, maximum pseudolikelihood is more precise, returning $1/4$ as many extraneous correlations.
We interpret this improved accuracy as a benefit of computing the gradient for a data point exactly, as opposed to using Gibbs sampling to estimate it as in maximum likelihood estimation.

\begin{figure}[t]
\begin{center}
\centerline{\includegraphics[width=.75\columnwidth]{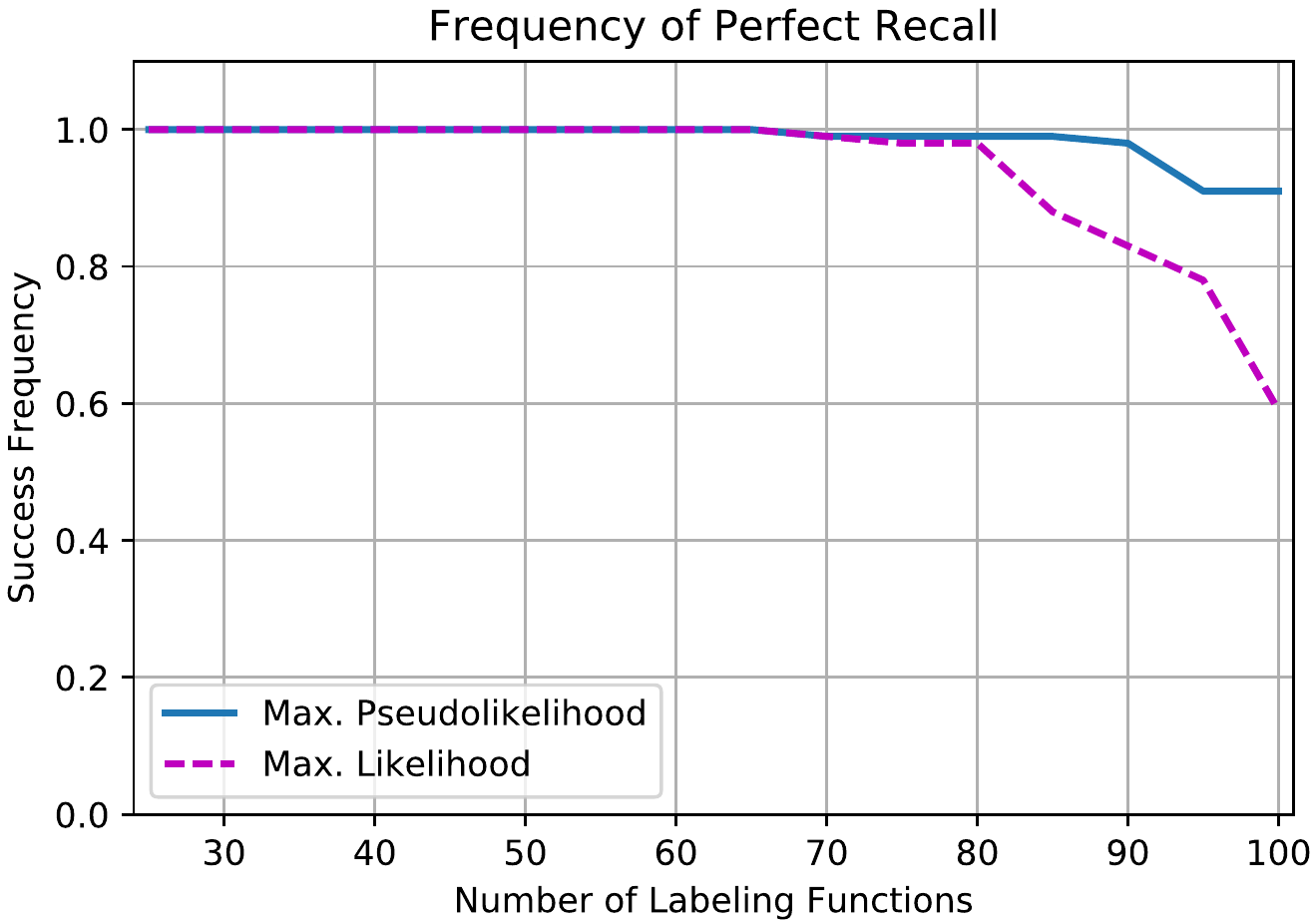}}
\vskip 0.1in
\centerline{\includegraphics[width=.75\columnwidth]{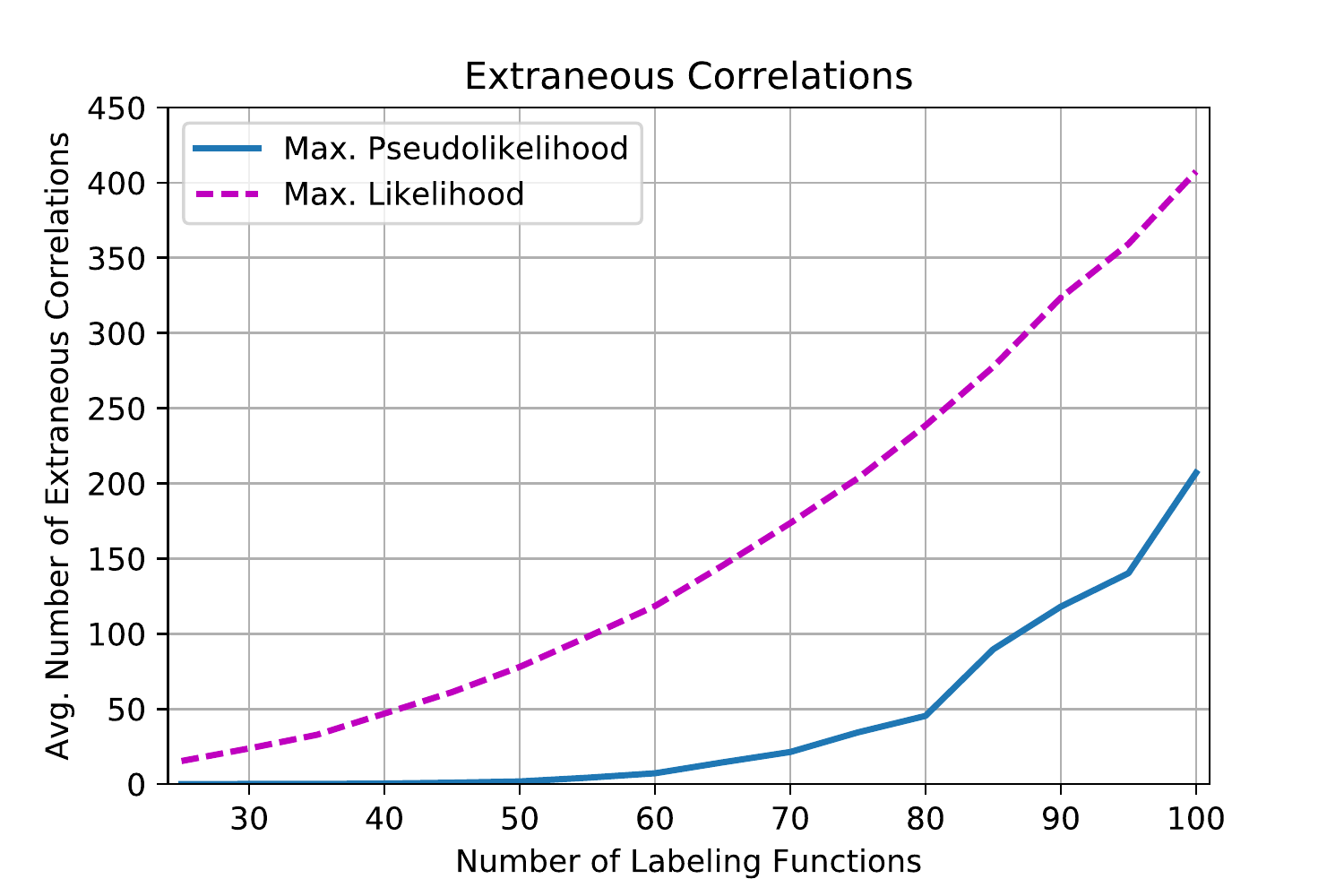}}
\caption{Comparison of structure learning with using maximum likelihood parameter estimation to select a  model structure.
Even when tuned for better recall (top), structure learning is also more precise, returning $1/4$ as many extraneous correlations (bottom).}
\label{fig:comparison_accuracy}
\end{center}
\vskip -0.2in
\end{figure}

\subsection{Real-World Applications}
\label{sec:realworld}

We evaluate our method on several real-world information extraction applications, comparing 
the performance of data programming using 
dependencies selected by our method with the conditionally independent model (Table~\ref{tab:apps}).
In the data programming method, users express a variety of weak supervision rules and sources such as 
regular expression patterns, distant supervision from dictionaries and existing knowledge bases, 
and other heuristics as labeling functions.
Due to the noisy and overlapping nature of these labeling functions,
correlations arise.
Learning this correlation structure gives an average improvement of 1.5 F1 points.

Extracting structured information from unstructured text by identifying mentioned entities and relations is a challenging task that is well studied in the context of weak supervision~\cite{bunescu:acl07, alfonseca:acl12, ratner:nips16}.
We consider three tasks: extracting mentions of specific diseases from
the scientific literature (\textit{Disease Tagging}); extracting mentions of
chemicals inducing diseases from the scientific literature (\textit{Chemical-Disease});
and extracting mentions of electronic device polarity from PDF parts sheet tables
(\textit{Device Polarity}).
In the first two applications, we consider a training set of 500 unlabeled abstracts from PubMed,
and in the third case 100 PDF parts sheets consisting of mixed text and tabular data.
We use hand-labeled test sets to evaluate on the candidate-mention-level performance, which is the accuracy of the classifier in identifying correct mentions of
specific entities or relations, given a set of candidate mentions.
For example, in Chemical-Disease, we consider as candidates all pairs of co-occurring chemical-disease mention pairs as identified by standard preprocessing tools\footnote{\tiny{\url{ncbi.nlm.nih.gov/CBBresearch/Lu/Demo/PubTator/index.cgi}}}.

We see that modeling the correlations between labeling functions gives gains in performance which appear to be correlated with the total number of sources.
For example, in the disease tagging application, we have 233 labeling functions, the majority of which check for membership
in specific subtrees of a reference disease ontology using different matching heuristics.
There is overlap in the labeling functions which check identical subtrees of the ontology, and we see that our method
increases end performance by a significant 2.6 F1 points by modeling this structure.

Examining the Chemical-Disease task, we see that our method identifies correlations that are
both obviously true and ones that are more subtle.
For example, our method learns dependencies between labeling functions that are
compositions of one another, such as one labeling function checking for the pattern \texttt{[CHEM] induc.* [DIS]},
and a second labeling function checking for this pattern plus membership in an external knowledge base of known
chemical-disease relations.
Our method also learns more subtle correlations: for example, it selected a correlation
between a labeling function that checks for the presence of a chemical mention in between the chemical and disease mentions comprising the candidate, and one that checks for the pattern \texttt{.*-induced} appearing in between.

\begin{table*}[t]
\caption{Candidate-mention scores of information extraction applications trained with data programming using generative models with no dependency structure (\textit{Independent}) and learned dependency structure
(\textit{Structure}).}
\label{tab:apps}
\vskip 0.15in
\begin{center}
\begin{small}
\begin{sc}
\begin{tabular}{lccccccccccccc}
\toprule
\multirow{2}{*}{Application} & \phantom{} & \multicolumn{3}{c}{Independent} & \phantom{} & \multicolumn{3}{c}{Structure} & \phantom{} & \multirow{2}{*}{F1 Diff.} & \multirow{2}{*}{\# LFs} & \multirow{2}{*}{\# Cor.} & \multirow{2}{*}{\% Corr.} \\
& & P & R & F1 & & P & R & F1 & & & & & \\
\midrule
\belowspace
Disease Tagging & & 60.4 & 73.3 & 66.3 & & 68.0 & 69.8 & 68.9 & & 2.6 & 233 & 315 & \phantom{0}1.17\%\\
\belowspace
Chemical-Disease & & 45.1 & 69.2 & 54.6 & & 46.8 & 69.0 & 55.9 & & 1.3 & 33 & 21 & \phantom{0}3.98\%\\
Device Polarity & & 78.9 & 99.6 & 88.1 & & 80.5 & 98.6 & 88.7 & & 0.6 & 12 & 32 & 48.49\%\\
\bottomrule
\end{tabular}
\end{sc}
\end{small}
\end{center}
\vskip -0.1in
\end{table*}

\subsection{Accelerating Application Development}
\label{sec:catastrophe}

Our method is in large part motivated by the new programming model
introduced by weak supervision, and the novel hurdles that developers face.
For example in the Disease Tagging application above,
we observed developers significantly slowed down in trying to 
to leverage the rich disease ontologies and matching heuristics they had available
without introducing too many dependencies between their labeling functions.
In addition to being slowed down, we also observed
developers running into significant pitfalls due to unnoticed correlations between their
weak supervision sources.
In one collaborator's application, for every labeling function that referenced the words
in a sentence, a corresponding labeling function referenced the lemmas, which were often identical, and this 
significantly degraded performance.
By automatically learning dependencies, we were able to significantly
mitigate the effects of such correlations.
We therefore envision an accelerated development process enabled by our method.

To further explore the way in which our method can protect against such types of failure modes, we consider adding correlated, random labeling functions to those used in the Chemical-Disease task.
Figure~\ref{fig:catastrophe} shows the average estimated accuracy of copies of a random labeling function.
An independent model grows more confident that the random noise is accurate.
However, with structure learning, we identify that the noisy sources are not independent and they therefore do not outvote the real labeling functions.
In this way, structure learning can protect against failures as users experiment with sources of weak supervision.

\begin{figure}[t]
\begin{center}
\centerline{\includegraphics[width=.75\columnwidth]{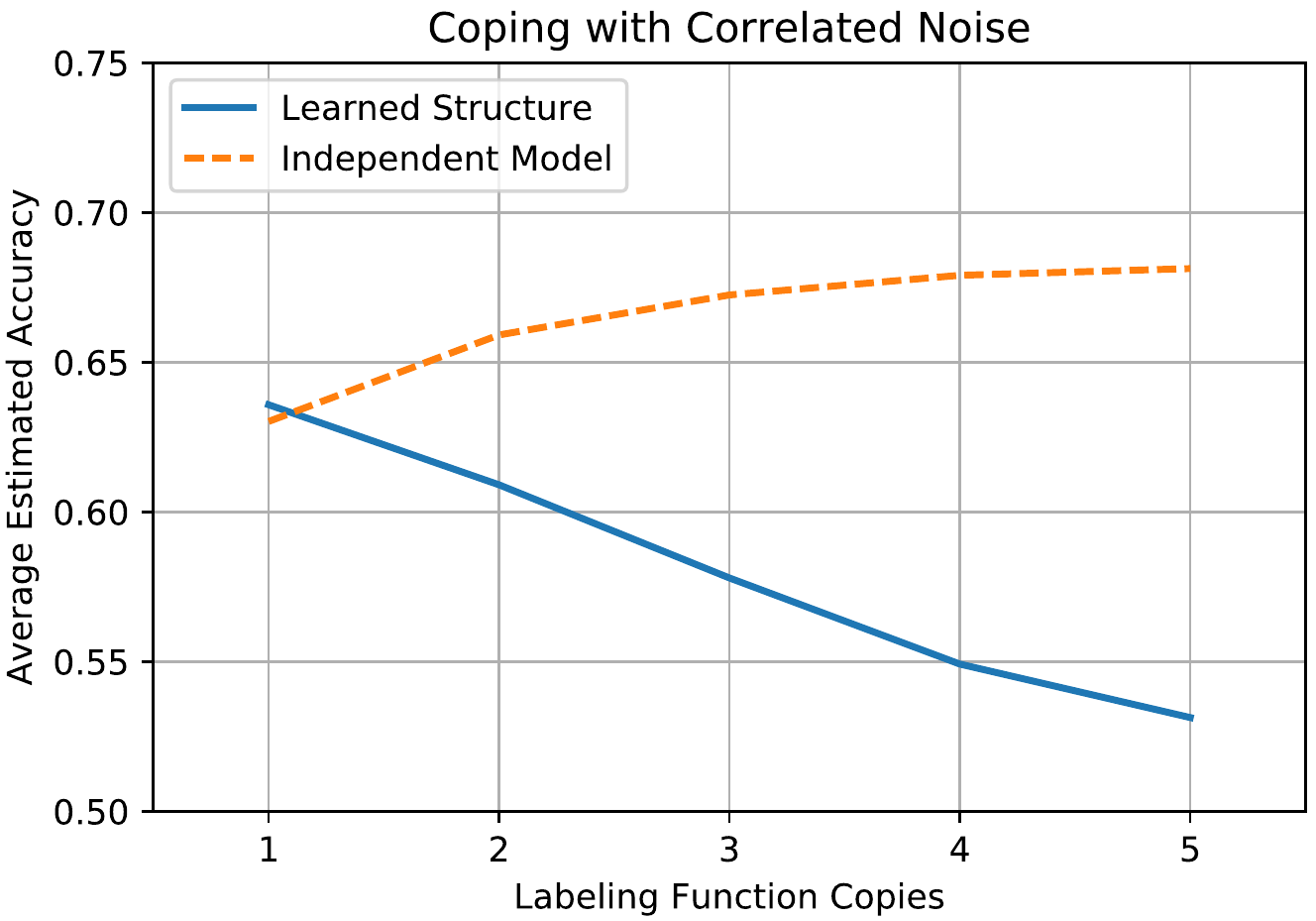}}
\caption{Structure learning identifies and corrects correlated, random labeling functions added to the Chemical-Disease task.}
\label{fig:catastrophe}
\end{center}
\vskip -0.4in
\end{figure}


\section{Related Work}
\label{sec:related}

Structure learning is a well-studied problem, but most work has assumed access to hand-labeled training data.
Some of the earliest work has focused on generalized linear models.
The lasso \citep{tibshirani:royalstats96}, linear regression with $\ell_1$ regularization, is a classic technique.
\citet{zhao:jmlr06} showed that the lasso is a consistent structure estimator.
The Dantzig selector \citep{candes:annalsofstats07} is another structure estimator for linear models that uses $\ell_1$, which can learn in the high-dimensional setting where there are more possible dependencies than samples.
\citet{ng:icml04} showed that $\ell_1$-regularized logistic regression has sample complexity logarithmic in the number of features.
$\ell_1$ regularization has also been used as a prior for compressed sensing \citep[e.g.,][]{donoho:pnas03, tropp:infotheory06, wainwright:infotheory09}.

Regularized estimators have also been used to select structures for graphical models.
\citet{meinshausen:annalsofstats06} showed that parameter learning with $\ell_1$ regularization for Gaussian graphical models under similar assumptions also consistently selects the correct structure.
Most similar to our proposed estimator, \citet{ravikumar:annalsofstats10} propose a fully supervised pseudolikelihood estimator for Ising models.
Also related is the work of \citet{chandrasekaran:annalsofstats12}, which considers learning the structure of Gaussian graphical models with latent variables.
Other techniques for learning the structure of graphical models include grafting \citep{perkins:jmlr03, zhu:kdd10} and the information bottleneck approach for learning Bayesian networks with latent variables \citep{elidan:jmlr05}.

Using heuristic sources of labels is increasingly common.
Treating labels from a single heuristic source as gold labels is called distant supervision \citep{craven:ismb99, mintz:acl09}.
Some methods use multi-instance learning to reduce the noise in a distant supervision source \citep{riedel:ecml10, hoffmann:acl11}.
Others use hierarchical topic models to generate additional training data for weak supervision, but they do not support user-provided heuristics \citep{alfonseca:acl12, takamatsu:acl12, roth:cikm13, roth:emnlp13}.
Previous methods that support heuristics for weak supervision \citep[e.g.,][]{bunescu:acl07, shin:vldb15} do not model the noise inherent in these sources.
Also, \citet{downey:nips08} showed that PAC learning is possible without hand-labeled data if the features monotonically order data by class probability.

Estimating the accuracy of multiple label sources without a gold standard is a classic problem \citep{dawid:royalstats79}, and many proposed approaches are generalized in the data programming framework.
\citet{parisi:pnas14} proposed a spectral approach to estimating the accuracy of members of classifier ensembles.
Many methods for crowdsourcing estimate the accuracy of workers without hand-labeled data \citep[e.g.,][]{dalvi:www13, joglekar:icde15, zhang:jmlr16}.
In data programming, the scaling of data to label sources is different from crowdsourcing; a relatively small number of sources label all the data.
We can therefore learn rich dependency structures among the sources.

\vspace{-.4em}

\section{Conclusion and Future Directions}

We showed that learning the structure of a generative model enables higher quality data programming results.
Our method for structure learning is also 100$\times$ faster than a maximum likelihood approach.
If data programming and other forms of weak supervision are to make machine learning tools easier to develop, selecting accurate structures for generative models with minimal user intervention is a necessary capability.
Interesting questions remain.
Can the guarantee of Theorem~\ref{thmsparsistency} be tightened for higher-order dependencies to match the pairwise case of Corollary~\ref{corsublinear}?
Preliminary experiments show that they converge at similar rates in practice.

\section*{Acknowledgements}

Thanks to Christopher De Sa for helpful discussions, and Henry Ehrenberg and Sen Wu for assistance with experiments.
We gratefully acknowledge the support of the Defense Advanced Research Projects Agency (DARPA) SIMPLEX program under No. N66001-15-C-4043, the DARPA D3M program under No. FA8750-17-2-0095, the National Science Foundation (NSF) CAREER Award under No. IIS- 1353606, the Office of Naval Research (ONR) under awards No. N000141210041 and No. N000141310129, a Sloan Research Fellowship, the Moore Foundation, an Okawa Research Grant, Toshiba, and Intel. Any opinions, findings, and conclusions or recommendations expressed in this material are those of the authors and do not necessarily reflect the views of DARPA, NSF, ONR, or the U.S. government.

\bibliography{bach-icml17}
\bibliographystyle{icml2017}

\clearpage

\onecolumn
\appendix

\section{Proofs}

In this appendix, we provide proofs for Theorem~\ref{thmsparsistency} and Corollary~\ref{corsublinear} from the main text.
In Section~\ref{app: outline}, we provide an outline of the proof and state several lemmas.
In Section~\ref{app: thm}, we prove Theorem~\ref{thmsparsistency}.
In Section~\ref{app: cor}, we prove Corollary~\ref{corsublinear}, which follows directly from Theorem~\ref{thmsparsistency}.
In Section~\ref{app: lem}, we prove the lemmas stated in Section~\ref{app: outline}.

\subsection{Outline and Lemma Statements}
\label{app: outline}

\subsubsection{Outline of Theorem \ref{thmsparsistency} Proof}
We first show that the negative marginal log-pseudolikelihood is strongly convex under condition \eqref{eq: cov}.
In particular, in Lemma~\ref{lemma:grad_hessian}, we derive the gradient and Hessian of each term of the negative marginal log-pseudolikelihood,
and in Lemma~\ref{lemma:convex}, we show that the negative marginal log-pseudolikelihood is strongly convex under condition \eqref{eq: cov}.

Next, in Lemma~\ref{lemma:small}, we show that, under condition \eqref{eq: well-specified}, the gradient of the negative marginal log-pseudolikelihood at the true parameter $\param^*$ is small with high probability.

Finally, we show that if we run SGD until convergence and then truncate, we will recover the exact sparsity structure with high probability.
In Lemma~\ref{lemma:close}, we show that if the true parameter $\param^*$ has a small gradient, then the empirical minimum $\hat \param$ will be close to it, and in Lemma~\ref{lemma:truncate}, we show that the correct sparsity structure is recovered.

\subsubsection{Lemma Statements}
We now formally state the lemmas used in our proof.
\begin{restatable}{lemma}{lemmagradhessian}
  \label{lemma:grad_hessian}
  Given a family of maximum-entropy distributions
  \begin{align*}
    \p_\param(x) = \frac{1}{Z_\param}\exp(\param^T\dep(x)),
  \end{align*}
  for some function of sufficient statistics $h\colon \Omega \rightarrow \mathbb{R}^M$, if we let $J$ be the negative log-pseudolikelihood objective for some event $A\subseteq \Omega$,
  \begin{align*}
    J(\param) = -\log p_{x\sim\p_\param}(x\in A\mid \lfout_{\setminus \ilf}),
  \end{align*}
  then its gradient is
  \begin{align*}
    \nabla J(\theta) = -\mathbb{E}_{x\sim\p_\theta}\left[\dep(x)\mid x\in A, \lfout_{\setminus \ilf}\right]
                     + \mathbb{E}_{x\sim\p_\theta}\left[\dep(x)\mid \lfout_{\setminus \ilf}\right]
  \end{align*}
  and its Hessian is
  \begin{align*}
    \nabla^2 J(\theta) = -\Cov_{x\sim\p_\theta}\left[\dep(x)\mid x\in A, \lfout_{\setminus \ilf}\right]
                       + \Cov_{x\sim\p_\theta}\left[\dep(x)\mid \lfout_{\setminus \ilf}\right]
  \end{align*}
\end{restatable}

\begin{restatable}{lemma}{lemmaconvex}
  \label{lemma:convex}
  Let $J$ be the empirical negative log-pseudolikelihood objective for the event $\lfout_\ilf = \lfoutobs_\ilf$
  \begin{align*}
    J(\param) = -\sum_{\iy=1}^{\ny}\left[\log p_{x\sim\p_\param}(\lfout_{\iy\ilf} = \lfoutobs_{\iy\ilf} \mid \lfout_{\iy\setminus\ilf} = \lfoutobs_{\iy\setminus\ilf})\right].
  \end{align*}
  Let $\Param_\ilf$ denote the set of parameters corresponding to dependencies incident on either labeling function $\lf_\ilf$ or the true label $\y$, and
  let $\Param_{\setminus\ilf}$ denote all the set of all remaining parameters.

  Then, $J(\param)$ is independent of the variables in $\Param_{\setminus \ilf}$, and under condition \eqref{eq: cov}, $J(\param)$ is strongly convex on the variables in $\Param_\ilf$ with a parameter of strong convexity of $c$.
\end{restatable}

\begin{restatable}{lemma}{lemmasmall}
  \label{lemma:small}
  Let $d_\ilf$ be the number of dependencies that involve either $\lf_\ilf$ or $\y$, and let $\param^*$ be the true parameter specified by condition \eqref{eq: well-specified}.
  Define $W$ as the gradient of the negative log-pseudolikelihood of $\lf_\ilf$ at this point
  \begin{align*}
    W = -\nabla J(\param^*;X).
  \end{align*}

  Then, for any $\delta$.
  \begin{align*}
    \Pr(\|W\|_\infty \geq \delta) \leq 2d_\ilf\exp\left(\frac{-m\delta^2}{8}\right)
  \end{align*}
\end{restatable}

\begin{restatable}{lemma}{lemmaclose}
  \label{lemma:close}
  Let $J$ be a $c$-strongly convex function in $d$ dimensions, and let $\hat\param$ be the minimizer of $J$.
  Suppose $\|J(\param^*)\|_\infty \leq \delta$.
  Then,
  \begin{align*}
    \|\hat\param - \param^*\|_\infty \leq \frac{\delta}{c}\sqrt{d}
  \end{align*}
\end{restatable}

\begin{restatable}{lemma}{lemmatruncate}
  \label{lemma:truncate}
  Suppose that conditions \eqref{eq: well-specified}, \eqref{eq: cov}, and \eqref{eq: notzero} are satisfied.
  Suppose we run Algorithm~\ref{alg:method} with $\ny$ samples, a sufficiently small step size $\stepsize$, a sufficiently large number of epochs $\epochs$, and truncate once at the end with $\truncation\stepsize\threshold = \kappa/2$.
  Then, the probability that we fail to recover the exact sparsity structure is at most
  \begin{align*}
    2\nlf d\exp\left(\frac{-mc^2\kappa^2}{32d}\right).
  \end{align*}
\end{restatable}

\subsection{Proof of Theorem~\ref{thmsparsistency}}
\label{app: thm}

\thmsparsistency*
\begin{proof}
  If follows from Lemma~\ref{lemma:truncate} that the probability that we fail to recover the sparsity structure is at most
  \begin{align*}
    2\nlf d\exp\left(\frac{-mc^2\kappa^2}{32d}\right).
  \end{align*}

  By using the provided $\ny$, the probability of failure is at most
  \begin{align*}
    2\nlf d\exp\left(\frac{-\frac{32d}{c^2\kappa^2}\log\left(\frac{2nd}{\delta}\right)c^2\kappa^2}{32d}\right)
    &=
    2\nlf d\exp\left(-\log\left(\frac{2nd}{\delta}\right)\right)
    =
    \delta.
  \end{align*}

  Thus, we will succeed with probability at least
    $1 - \delta$.
\end{proof}

\subsection{Proof of Corollary~\ref{corsublinear}}
\label{app: cor}

\corsublinear*
\begin{proof}
  In this case, each labeling function $\lf_\ilf$ is involved in $\nlf - 1$ with other labeling functions, and the true label $\y$ is involved in $\nlf$ dependencies.
  Thus, $d = (\nlf - 1) + \nlf < 2\nlf$.

  We can then apply Theorem~\ref{thmsparsistency} to show that the probability of success is at least $1 - \tau$ for the specified $\ny$.
\end{proof}

\subsection{Proofs of Lemmas}
\label{app: lem}

\lemmagradhessian*
\begin{proof}
  We first rewrite the netative log-pseudolikelihood as
  \begin{align*}
    J(\param) &= -\log\Pr_{x\sim\pi_\param}(x\in A\mid \lfout_{\setminus \ilf}) \\
              &= -\log\frac{\Pr_{x\sim\pi_\param}(x\in A, \lfout_{\setminus\ilf})}{\Pr_{x\sim\pi_\param}(\lfout_{\setminus \ilf})} \\
              &= -\log\frac{\sum_{x\in A, \lfout_{\setminus\ilf}}\p_\param(x)}{\sum_{x\in\lfout_{\setminus \ilf}}p_\param(x)} \\
              &= -\log\frac{\sum_{x\in A, \lfout_{\setminus\ilf}}\exp(\param^T\dep(x))}{\sum_{x\in\lfout_{\setminus \ilf}}\exp(\param^T\dep(x))} \\
              &= -\log\sum_{x\in A, \lfout_{\setminus\ilf}}\exp(\param^T\dep(x)) + \log \sum_{x\in\lfout_{\setminus \ilf}}\exp(\param^T\dep(x)).
  \end{align*}

  We now derive the gradient
  \begin{align*}
    \nabla J(\param)
              &= \nabla \left[-\log\sum_{x\in A, \lfout_{\setminus\ilf}}\exp(\param^T\dep(x)) + \log \sum_{x\in\lfout_{\setminus\ilf}}\exp(\param^T\dep(x))\right] \\
              &= -\nabla\log\sum_{x\in A, \lfout_{\setminus\ilf}}\exp(\param^T\dep(x)) + \nabla\log \sum_{x\in\lfout_{\setminus\ilf}}\exp(\param^T\dep(x)) \\
              &= - \frac{\sum_{x\in A, \lfout_{\setminus\ilf}}\dep(x)\exp(\param^T\dep(x))}{\sum_{x\in A, \lfout_{\setminus\ilf}}\exp(\param^T\dep(x))}
                 + \frac{\sum_{x\in\lfout_{\setminus\ilf}}\dep(x)\exp(\param^T\dep(x))}{\sum_{x\in\lfout_{\setminus\ilf}}\exp(\param^T\dep(x))} \\
              &= -\mathbb{E}_{x\sim\p_\param}\left[\dep(x)\mid x\in A, \lfout_{\setminus\ilf}\right]
                 +\mathbb{E}_{x\sim\p_\param}\left[\dep(x)\mid \lfout_{\setminus\ilf}\right]
  \end{align*}

  We now derive the Hessian
  \begin{align*}
    \nabla^2 J(\param)
              &= \nabla\left[- \frac{\sum_{x\in A, \lfout_{\setminus\ilf}}\dep(x)\exp(\param^T\dep(x))}{\sum_{x\in A, \lfout_{\setminus\ilf}}\exp(\param^T\dep(x))}
                 + \frac{\sum_{x\in\lfout_{\setminus\ilf}}\dep(x)\exp(\param^T\dep(x))}{\sum_{x\in\lfout_{\setminus\ilf}}\exp(\param^T\dep(x))}\right] \\
              &= -\nabla\frac{\sum_{x\in A, \lfout_{\setminus\ilf}}\dep(x)\exp(\param^T\dep(x))}{\sum_{x\in A, \lfout_{\setminus\ilf}}\exp(\param^T\dep(x))}
                 +\nabla\frac{\sum_{x\in\lfout_{\setminus\ilf}}\dep(x)\exp(\param^T\dep(x))}{\sum_{x\in\lfout_{\setminus\ilf}}\exp(\param^T\dep(x))} \\
              &= -\left(\frac{\sum_{x\in A, \lfout_{\setminus\ilf}}\dep(x)\dep(x)^T\exp(\param^T\dep(x))}{\sum_{x\in A, \lfout_{\setminus\ilf}}\exp(\param^T\dep(x))}
-
    \frac{\left(\sum_{x\in A, \lfout_{\setminus\ilf}}\dep(x)\exp(\param^T\dep(x))\right)\left(\sum_{x\in A, \lfout_{\setminus\ilf}}\dep(x)\exp(\param^T\dep(x))\right)^T}{\left(\sum_{x\in A, \lfout_{\setminus\ilf}}\exp(\param^T\dep(x))\right)^2}
    \right) \\
              &\phantom{{}={}} +\left(\frac{\sum_{x\in \lfout_{\setminus\ilf}}\dep(x)\dep(x)^T\exp(\param^T\dep(x))}{\sum_{x\in \lfout_{\setminus\ilf}}\exp(\param^T\dep(x))}
-
    \frac{\left(\sum_{x\in \lfout_{\setminus\ilf}}\dep(x)\exp(\param^T\dep(x))\right)\left(\sum_{x\in \lfout_{\setminus\ilf}}\dep(x)\exp(\param^T\dep(x))\right)^T}{\left(\sum_{x\in \lfout_{\setminus\ilf}}\exp(\param^T\dep(x))\right)^2}
    \right) \\
              &= -\left(\mathbb{E}_{x\sim\p_\param}\left[\dep(x)\dep(x)^T\mid x\in A, \lfout_{\setminus\ilf}\right]
              -\mathbb{E}_{x\sim\p_\param}\left[\dep(x)\mid x\in A, \lfout_{\setminus\ilf}\right]
              \mathbb{E}_{x\sim\p_\param}\left[\dep(x)\mid x\in A, \lfout_{\setminus\ilf}\right]^T\right) \\
               &\phantom{{}={}}+\left(\mathbb{E}_{x\sim\p_\param}\left[\dep(x)\dep(x)^T\mid x\in \lfout_{\setminus\ilf}\right]
              -\mathbb{E}_{x\sim\p_\param}\left[\dep(x)\mid x\in \lfout_{\setminus\ilf}\right]
              \mathbb{E}_{x\sim\p_\param}\left[\dep(x)\mid x\in \lfout_{\setminus\ilf}\right]^T\right) \\
              &= -\Cov_{x\sim\p_\param}\left[\dep(x)\mid x\in A, \lfout_{\setminus\ilf}\right]
                 +\Cov_{x\sim\p_\theta}\left[\dep(x)\mid \lfout_{\setminus\ilf}\right].
  \end{align*}
\end{proof}

\lemmaconvex*
\begin{proof}
  First, we show that $J(\param)$ is independent of the variables in $\Param_{\setminus \ilf}$.
  We simplify $J(\param)$ as
  \begin{align*}
    J(\param)
    &= \sum_{\iy=1}^{\ny}\left[-\log\sum_{x\in \lfoutobs}\exp(\param^T\dep(x)) + \log \sum_{x\in\lfoutobs_{\setminus \ilf}}\exp(\param^T\dep(x))\right] \\
    &= \sum_{\iy=1}^{\ny}\left[
      - \log \sum_{x\in \lfoutobs}\exp\left(\param_\ilf^T\dep_\ilf(x) + \param_{\setminus\ilf}^T\dep_{\setminus\ilf}(x)\right)
      + \log \sum_{x\in\lfoutobs_{\setminus \ilf}}\exp\left(\param_\ilf^T\dep_\ilf(x) + \param_{\setminus\ilf}^T\dep_{\setminus\ilf}(x)\right)
      \right] \\
    &= \sum_{\iy=1}^{\ny}\left[
      - \log \sum_{x\in \lfoutobs}\left(\exp\left(\param_\ilf^T\dep_\ilf(x)\right)\exp\left(\param_{\setminus\ilf}^T\dep_{\setminus\ilf}(x)\right)\right)
      + \log \sum_{x\in\lfoutobs_{\setminus \ilf}}\left(\exp\left(\param_\ilf^T\dep_\ilf(x)\right)\exp\left(\param_{\setminus\ilf}^T\dep_{\setminus\ilf}(x)\right)\right)
      \right] \\
    &= \sum_{\iy=1}^{\ny}\left[
      - \log \left(\exp\left(\param_{\setminus\ilf}^T\dep_{\setminus\ilf}(x)\right)\sum_{x\in \lfoutobs}\exp\left(\param_\ilf^T\dep_\ilf(x)\right)\right)
      + \log \left(\exp\left(\param_{\setminus\ilf}^T\dep_{\setminus\ilf}(x)\right)\sum_{x\in\lfoutobs_{\setminus \ilf}}\exp\left(\param_\ilf^T\dep_\ilf(x)\right)\right)
      \right] \\
    &= \sum_{\iy=1}^{\ny}\left[
      - \log \exp\left(\param_{\setminus\ilf}^T\dep_{\setminus\ilf}(x)\right) - \log \sum_{x\in \lfoutobs}\exp\left(\param_\ilf^T\dep_\ilf(x)\right)
      + \log \exp\left(\param_{\setminus\ilf}^T\dep_{\setminus\ilf}(x)\right) + \log \sum_{x\in\lfoutobs_{\setminus \ilf}}\exp\left(\param_\ilf^T\dep_\ilf(x)\right)
      \right] \\
    &= \sum_{\iy=1}^{\ny}\left[
      - \log \sum_{x\in \lfoutobs}\exp\left(\param_\ilf^T\dep_\ilf(x)\right)
      + \log \sum_{x\in\lfoutobs_{\setminus \ilf}}\exp\left(\param_\ilf^T\dep_\ilf(x)\right)
      \right],
  \end{align*}
  which does not depend on any variables in $\Param_{\setminus \ilf}$.

  Next, we prove that $J(\param)$ is $c$-strongly convex in the variabes in $\Param_{\ilf}$.
  By combining the previous result and Lemma \ref{lemma:grad_hessian}, we can derive the Hessian
  \begin{align*}
    \nabla^2 J(\Param_\ilf) 
      &=
      \sum_{\iy=1}^{\ny}\left[-\Cov_{(\lfout, Y)\sim p_{\theta}}(\Dep_\ilf(\lfout, Y)\mid \lfout_\iy = \lfoutobs_\iy)
      +
            \Cov_{(\lfout, Y)\sim p_{\theta}}(\Dep_\ilf(\lfout, Y)\mid \lfout_{\iy\setminus \ilf} = \lfoutobs_{\iy\setminus \ilf})\right] \\
      &=
      -\sum_{\iy=1}^{\ny}\Cov_{(\lfout, Y)\sim p_{\theta}}(\Dep_\ilf(\lfout, Y)\mid \lfout_\iy = \lfoutobs_\iy)
      +
            \sum_{\iy=1}^{\ny}\Cov_{(\lfout, Y)\sim p_{\theta}}(\Dep_\ilf(\lfout, Y)\mid \lfout_{\iy\setminus \ilf} = \lfoutobs_{\iy\setminus \ilf}).
  \end{align*}
  It then follows from condition \eqref{eq: cov} that
  \begin{align*}
    cI\preceq\nabla^2 J(\Param_\ilf),
  \end{align*}
  which implies that $J$ is $c$-strongly convex on variables in $\Theta_j$.
\end{proof}

\lemmasmall*
\begin{proof}
  From Lemma~\ref{lemma:grad_hessian}, we know that each element of $W$ can be written as the average of $\ny$ i.i.d. terms.
  From condition~\eqref{eq: cov}, we know that the terms have zero mean, and we also know that the terms are bounded in absolute value by 2, due to the fact that the dependencies have values falling in the interval $[-1,1]$.

  We can alternatively think of the average of the terms as the sum of $\ny$ i.i.d. zero-mean random variables that are bounded in absolute value by $\frac{2}{\ny}$.
  The two-sided Azuma's inequality bounds the probability that any term in $W$ is large.
  \begin{gather*}
    \Pr(|W_\ilf| \geq \delta)
    \leq 2\exp\left(\frac{-\delta^2}{2\sum_{\iy=1}^{\ny}\left(\frac{2}{\ny}\right)^2}\right)
    \leq 2\exp\left(\frac{-m\delta^2}{8}\right)
  \end{gather*}

  The union bound then bounds the probability that any component of $W$ is large.
  \begin{gather*}
    \Pr(\|W\|_\infty \geq \delta) \leq 2d_\ilf\exp\left(\frac{-m\delta^2}{8}\right)
  \end{gather*}
\end{proof}

\lemmaclose*
\begin{proof}
  Because $J$ is $c$-strongly convex,
  \begin{gather*}
    \left(\nabla J(\param^*) - \nabla J(\hat \param) \right)^T(\param^* - \hat \param) \geq c \|\param^* - \hat \param\|_2^2.
  \end{gather*}

  $\nabla J(\hat \param) = \mathbf{0}$, so
  \begin{gather*}
    \nabla J(\param^*)^T(\param^* - \hat \param) \geq c \|\param^* - \hat \param\|_2^2.
  \end{gather*}

  Then, because $\|J(\param^*)\|_\infty \leq \delta$,
  \begin{gather*}
    c \|\param^* - \hat \param\|_2^2 \leq \delta\|\param^* - \hat \param\|_1 \\
    \|\param^* - \hat \param\|_2^2 \leq \frac{\delta}{c}\|\param^* - \hat \param\|_1.
  \end{gather*}

  Then, we have that
  \begin{gather*}
    \|\param^* - \hat \param\|_2^2 \leq \frac{\delta^2}{c^2}d
    \|\param^* - \hat \param\|_2 \leq \frac{\delta}{c}\sqrt{d},
  \end{gather*}
  which implies that
  \begin{gather*}
    \|\param^* - \hat \param\|_\infty \leq \frac{\delta}{c}\sqrt{d}.
  \end{gather*}
\end{proof}

\lemmatruncate*
\begin{proof}
  First, we bound the probability that we fail to correctly recover the dependencies involving $\lf_\ilf$.
  By Lemma~\ref{lemma:small}, we can bound the probability that the gradient is large at $\param^*$ by
  \begin{align*}
    \Pr\left(\|W\|_\infty \geq \frac{c\kappa}{2\sqrt{d}}\right) &\leq 2d\exp\left(\frac{-mc^2\kappa^2}{32d}\right).
  \end{align*}

  Notice that if $\|W\|_\infty \geq \frac{c\kappa}{2\sqrt{d}}$, then $\|\param^* - \hat\param\|_\infty \leq \kappa/2$.
  If then follows from Lemma~\ref{lemma:close} that
  \begin{align*}
    \Pr\left(\|\param^* - \hat\param\|_\infty \geq \kappa/2\right)
    \leq
    2d\exp\left(\frac{-mc^2\kappa^2}{32d}\right).
  \end{align*}

  If this is the case, then upon truncation, the correct dependencies will be recovered for $\lf_\ilf$.
  We now use the union bound to show that we will fail to recover the exact sparsity structure with probability at most
  \begin{align*}
    2\nlf d\exp\left(\frac{-mc^2\kappa^2}{32d}\right).
  \end{align*}
\end{proof}

\end{document}